\documentclass[11pt]{article}

\usepackage[preprint]{acl}
\usepackage{times}
\usepackage{latexsym}
\usepackage[T1]{fontenc}
\usepackage[utf8]{inputenc}
\usepackage{microtype}
\usepackage{inconsolata}
\usepackage{graphicx}
\usepackage{booktabs}
\usepackage{amsmath,amssymb,amsthm}
\usepackage{algorithm}
\usepackage{algpseudocode}
\usepackage{xspace}
\usepackage{tikz}
\usepackage{fontawesome5}
\usepackage{placeins}
\usetikzlibrary{arrows.meta,positioning}

\newcommand{\core}{\textnormal{\textsc{Core}}\xspace}

\newcommand{\nogoods}{\mathcal{N}}
\newcommand{\verify}{\mathsf{Verify}}
\newcommand{\finalcheck}{\mathsf{FinalCheck}}
\newcommand{\propose}{\mathsf{Propose}}
\newcommand{\contains}{\mathsf{Contains}}
\newtheorem{definition}{Definition}
\newtheorem{proposition}{Proposition}
\newtheorem{theorem}{Theorem}

\title{\core: Conflict-Oriented Reasoning Elimination for\\
Verifiable Language-Model Search}

\author{
Siyu Song$^{1}$,
Rui Xu$^{2}$,
Jia Lin$^{3}$,
Kai Liu$^{1}$,
Weifang Wang$^{1}$\\
$^{1}$School of Computer Science and Technology, University of Science and Technology of China\\
$^{2}$School of Information Science and Engineering, Sun Yat-sen University\\
$^{3}$College of Computer Science, Shenzhen University
}

\begin{document}
\maketitle

\begin{abstract}
Test-time reasoning systems often respond to failure by restarting or revising
the latest step, even when an earlier decision caused the error.  We introduce
\core, a search controller that requests a certified \emph{conflict core} from
a verifier, backjumps to the latest decision in that core, and caches the
conflict to avoid repeating it.  Under sound verification, finite branching
and depth, and exhaustive proposals, the uncapped search is complete and
never prunes a valid solution.  On 2,000 planted graph-coloring instances with matched
proposals and an exact verifier, \core reduces median verifier calls by
39.8\% at 30 variables and 35.0\% at 36 variables relative to chronological
repair; caching further improves on backjumping alone.  Across five reasoning
tasks, \core achieves 75.9\% mean success with Qwen2.5-7B-Instruct and 84.2\%
with Qwen3-8B, compared with 72.5\% and 81.8\% for Tree of Thoughts.  It
also uses fewer verifier calls and generated tokens on both backbones.  These
results show the value of using certified failure explanations to direct
language-model search.
\end{abstract}

\section{Introduction}

Chain-of-thought exposes intermediate computation \citep{wei2022chain}.
Test-time methods improve reasoning by sampling traces \citep{wang2023self},
decomposing problems \citep{zhou2023least}, or searching over thoughts
\citep{yao2023tree,besta2024graph}.  Yet after failure, systems usually restart,
edit the latest step, or reflect and retry
\citep{madaan2023selfrefine,shinn2023reflexion}.  If the error was introduced
earlier, such local repair revisits irrelevant decisions.

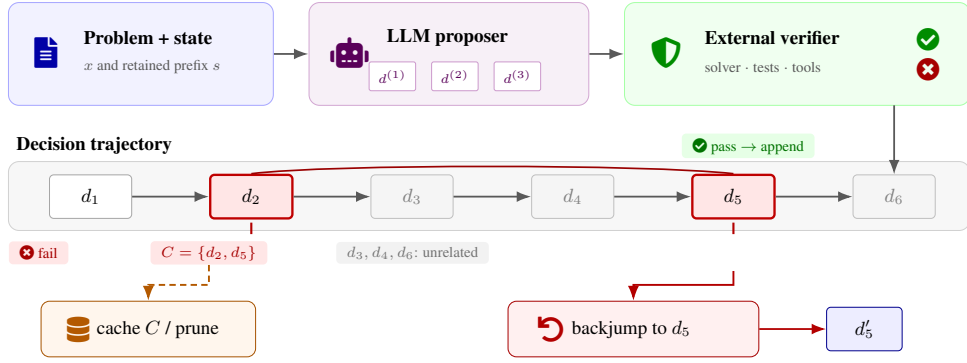
\begin{figure*}[t]
  \centering
  \resizebox{0.8\linewidth}{!}{\begin{tikzpicture}[
  x=.985cm,y=1cm,>=Latex,
  flow/.style={-Latex,line width=.85pt,draw=black!65},
  small/.style={font=\scriptsize},
  head/.style={font=\small\bfseries},
  token/.style={draw=black!35,rounded corners=2pt,fill=white,
    minimum width=1.36cm,minimum height=.72cm,font=\small},
  bad/.style={token,draw=red!75!black,fill=red!10,line width=1.1pt},
  ghost/.style={token,draw=black!25,fill=black!4,text=black!50},
  tag/.style={rounded corners=2pt,inner xsep=5pt,inner ysep=2pt,font=\scriptsize}
]
  \fill[blue!5,rounded corners=5pt] (0,4.35) rectangle (4.45,6.05);
  \fill[violet!6,rounded corners=5pt] (5.05,4.35) rectangle (9.75,6.05);
  \fill[green!6,rounded corners=5pt] (10.35,4.35) rectangle (16.15,6.05);
  \draw[blue!35,rounded corners=5pt] (0,4.35) rectangle (4.45,6.05);
  \draw[violet!35,rounded corners=5pt] (5.05,4.35) rectangle (9.75,6.05);
  \draw[green!35,rounded corners=5pt] (10.35,4.35) rectangle (16.15,6.05);

  \node[blue!65!black,font=\Large] at (.63,5.23) {\faIcon{file-alt}};
  \node[head,anchor=west] at (1.13,5.47) {Problem + state};
  \node[small,anchor=west,text=black!65] at (1.13,4.99) {$x$ and retained prefix $s$};

  \node[violet!70!black,font=\Large] at (5.73,5.23) {\faIcon{robot}};
  \node[head,anchor=west] at (6.22,5.47) {LLM proposer};
  \foreach \px/\value in {6.45/1,7.50/2,8.55/3}
    \node[draw=violet!40,fill=white,rounded corners=1pt,
      minimum width=.77cm,minimum height=.28cm,font=\tiny,
      text=violet!70!black] at (\px,4.83) {$d^{(\value)}$};

  \node[green!55!black,font=\Large] at (11.05,5.23) {\faIcon{shield-alt}};
  \node[head,anchor=west] at (11.56,5.47) {External verifier};
  \node[small,anchor=west,text=black!65] at (11.56,4.99) {solver $\cdot$ tests $\cdot$ tools};
  \node[green!55!black,font=\normalsize] at (15.45,5.44)
    {\faIcon{check-circle}};
  \node[red!65!black,font=\normalsize] at (15.45,4.94)
    {\faIcon{times-circle}};

  \draw[flow] (4.46,5.2) -- (5.04,5.2);
  \draw[flow] (9.76,5.2) -- (10.34,5.2);

  \node[head,anchor=west] at (0,3.68) {Decision trajectory};
  \node[tag,fill=green!10,text=green!45!black,anchor=east] at (13.55,3.68)
    {\faIcon{check-circle}\ pass $\rightarrow$ append};
  \fill[black!3,rounded corners=5pt] (0,2.28) rectangle (16.15,3.42);
  \draw[black!18,rounded corners=5pt] (0,2.28) rectangle (16.15,3.42);
  \node[token] (d1) at (1.38,2.84) {$d_1$};
  \node[bad] (d2) at (4.08,2.84) {$d_2$};
  \node[ghost] (d3) at (6.78,2.84) {$d_3$};
  \node[ghost] (d4) at (9.48,2.84) {$d_4$};
  \node[bad] (d5) at (12.18,2.84) {$d_5$};
  \node[ghost] (d6) at (14.88,2.84) {$d_6$};
  \foreach \a/\b in {d1/d2,d2/d3,d3/d4,d4/d5,d5/d6}
    \draw[flow] (\a.east) -- (\b.west);
  \draw[red!60!black,line width=.8pt]
    (d2.north) .. controls (5.4,3.38) and (10.7,3.38) .. (d5.north);
  \draw[flow] (14.88,4.35) -- (d6.north);
  \node[tag,fill=red!10,text=red!65!black,anchor=west] at (0,1.91)
    {\faIcon{times-circle}\ fail};
  \node[tag,fill=red!10,text=red!65!black,anchor=west] (coretag) at (2.38,1.91)
    {$C=\{d_2,d_5\}$};
  \node[tag,fill=black!5,text=black!55,anchor=west] at (5.5,1.91)
    {$d_3,d_4,d_6$: unrelated};
  \draw[red!70!black,line width=1pt] (4.08,2.41) -- (4.08,2.22);
  \draw[red!70!black,line width=1pt] (12.18,2.41) -- (12.18,2.22);

  \node[draw=orange!65!black,fill=orange!7,rounded corners=4pt,
    minimum width=3.55cm,minimum height=.92cm] (memory) at (2.35,.65) {};
  \node[orange!70!black,font=\large,anchor=west] at (.85,.65) {\faIcon{database}};
  \node[font=\small,anchor=west] at (1.35,.65) {cache $C$ / prune};
  \node[draw=red!65!black,fill=red!6,rounded corners=4pt,
    minimum width=4.15cm,minimum height=.92cm] (jump) at (10.5,.65) {};
  \node[red!70!black,font=\large,anchor=west] at (8.76,.65) {\faIcon{undo}};
  \node[font=\small,anchor=west] at (9.33,.65) {backjump to $d_5$};
  \node[token,draw=blue!55!black,fill=blue!7,minimum width=1.25cm]
    (retry) at (14.38,.65) {$d_5'$};
  \draw[flow,red!70!black] (jump.east) -- (retry.west);
  \draw[-Latex,densely dashed,orange!70!black,line width=.9pt]
    (coretag.south) -- (coretag.south |- 0,1.38) -| (memory.north);
  \draw[-Latex,red!70!black,line width=.9pt]
    (12.18,1.61) -- (12.18,1.38) -| (jump.north);
\end{tikzpicture}}
  \caption{Illustration of \core. A proposer extends a decision trajectory and
  an external verifier checks it. The highlighted decisions $d_2$ and $d_5$
  form a conflict core, while gray decisions are irrelevant to this failure.
  The controller learns the core, backjumps to $d_5$, and rejects future
  candidates that contain both conflicting decisions.}
  \label{fig:framework}
\end{figure*}

Classical solvers instead use \emph{conflict-directed search}: identify the
responsible decisions, jump over unrelated ones, and cache the conflict
\citep{dechter1990enhancement,prosser1993hybrid,silva1999grasp}.  Transferring
this principle to language models is nontrivial: natural-language ``thoughts''
are not solver literals, critiques need not be logically sufficient, and
approximate verifiers may reject valid reasoning.  A precise interface among
the proposer, verifier, and search controller is therefore required.

We propose \core (\textbf{C}onflict-\textbf{O}riented
\textbf{R}easoning \textbf{E}limination), a test-time algorithm for tasks whose
partial decisions can be checked.  The model proposes structured decisions
with stable identifiers.  On rejection, the verifier returns a \emph{conflict
core}: a subset of active decisions that cannot occur together in any valid
completion.  \core learns this core, backjumps to its most recent member, and
continues from there.  A core may come from a symbolic checker, an unsatisfiable
core, a failed program test with dynamic dependencies, or a learned verifier.
Only sound cores receive the paper's formal guarantee.
Figure~\ref{fig:framework} summarizes the separation between generation,
checking, and conflict-directed control.

Our contributions are:
\begin{itemize}
    \setlength{\topsep}{2pt}
    \setlength{\partopsep}{0pt}
    \setlength{\itemsep}{0pt}
    \setlength{\parsep}{0pt}
    \item We introduce a conflict-producing reasoning interface and \core,
    combining non-chronological backjumping with semantic nogood memoization.
    \item We prove conditional safety, completeness, and termination under
    explicit assumptions on verifier and conflict-core soundness.
    \item We provide a deterministic controlled testbed that isolates search
    under matched proposals and measures verifier calls, backtracks, and core
    reuse.
    \item We evaluate an LLM instantiation across five reasoning tasks and
    two backbones, comparing task success and inference cost with four baselines.
\end{itemize}

\section{Related Work}

\paragraph{Reasoning by sampling and search.}
Self-consistency aggregates independent traces \citep{wang2023self}, while
Tree of Thoughts, Graph of Thoughts, RAP, and LATS organize branching search
\citep{yao2023tree,besta2024graph,hao2023rap,zhou2024lats}.  MLR separates plan
descriptors from step execution \citep{xiong2026enhancing}.
These methods structure generation, choose states
to expand, or aggregate plans and answers; \core instead determines
\emph{where to return}
after a verified failure and which combinations to forbid thereafter.  The
methods are complementary: a best-first or tree-of-thought frontier can use
\core's learned constraints.

\paragraph{Verification and refinement.}
Process supervision scores intermediate steps
\citep{uesato2022solving,lightman2024verify}, and ProcessBench evaluates error
localization \citep{zheng2025processbench}.  Self-Refine, Reflexion, and CRITIC
use feedback for revision
\citep{madaan2023selfrefine,shinn2023reflexion,gou2024critic}.  Reasoning traces may be
unfaithful \citep{turpin2023unfaithful,lanham2023measuring}, so \core does not
assume that prose explanations reveal the model's internal computation.  It
requires only that accepted structured decisions and returned cores satisfy an
external contract.

\paragraph{Solver--LLM collaboration.}
LLM-Modulo systems separate candidate generation from reliable checking
\citep{kambhampati2024llmmodulo}.  SWAP represents dependencies with
entailment graphs and verifies intermediate steps
\citep{xiong2025deliberate}, while PAL executes model-generated programs
\citep{gao2023pal}.  Recent work also combines an LLM with
a solver and chronological backtracking for multi-constraint planning
\citep{stechly2025backtracking}.  \core contributes a different search rule:
conflicts name responsible decisions, enabling non-chronological backjumping
and persistent pruning.  Its lineage is conflict-directed backjumping and
constraint learning \citep{dechter1990enhancement,silva1999grasp}, reformulated
for a proposer whose actions are generated in language.

\section{Problem Formulation}

Let a problem instance be $x$.  A reasoning state
$s_t=(d_1,\ldots,d_t)$ is an ordered sequence of decisions.  A decision is a
record
\[
d_i=(\mathit{id}_i,\mathit{type}_i,\mathit{value}_i,\mathit{deps}_i),
\]
where the identifier is stable within a run and dependencies refer to prior
decisions.  Examples include assigning a value to a variable, selecting a
planning action, asserting a lemma, or choosing an API argument.  The proposer
$\propose_\theta(x,s_t)$ returns an ordered finite list of candidates.  It may
be an LLM, a policy model, or a deterministic heuristic.

Let $\kappa(d)$ be a canonical decision--value key that preserves the
decision's type, dependency context, and resource or occurrence identity
when repeated choices must remain distinct.  The active keys of a state are
\[
K(s_t)=\{\kappa(d_i):1\leq i\leq t\}.
\]
This representation lets the controller compare decisions across trajectories
while keeping choices with different dependencies distinct.  Figures abbreviate
$\kappa(d_i)$ as $d_i$.

The search-time verifier has three outputs:
\[
\verify(x,s_t \oplus d)=
\begin{cases}
\mathsf{pass},\\
\mathsf{fail}(C),
\mathsf{unknown},
\end{cases}
\]
where $C$ is a set of canonical decision--value keys drawn from the proposed
state.
\(\mathsf{fail}(C)\) is reserved for a certified sound core; heuristic
feedback cannot enter the hard cache.
\textsf{Unknown} means that the checker cannot establish either a valid step
or a sound conflict; search may retain the decision provisionally, but it
cannot learn a hard core from that outcome.  A separate
$\finalcheck(x,s)\in\{0,1\}$ scores a complete state against the task's final
criterion.

\begin{definition}[Sound conflict core]
Let $\mathcal{S}^{\star}(x)$ be the valid complete states for $x$.  A set $C$
is a sound conflict core for $x$ exactly when
\[
\forall s^{\star}\in\mathcal{S}^{\star}(x),\qquad
C\nsubseteq K(s^{\star}).
\]
\end{definition}

The definition is stronger than ``these steps look suspicious.''  It makes a
core a semantic nogood.  A direct constraint violation may yield a two-decision
core; a failed unit test may yield the decisions in its dynamic slice; a
symbolic solver may return an unsatisfiable core.  A learned verifier can
approximate this interface, but its cores are not covered by exact guarantees.

The objective is to find a complete state accepted by $\finalcheck$ within a
budget of proposer and verifier calls.  We count verifier calls because they
can involve tool execution, a second model, or a solver and frequently
dominate orchestration cost.

\section{\core}

\begin{figure*}[t]
  \centering
  \resizebox{0.8\linewidth}{!}{\begin{tikzpicture}[
  x=1cm,y=1cm,>=Latex,
  flow/.style={-Latex,line width=.85pt,draw=black!60},
  repair/.style={-Latex,line width=1pt,draw=blue!65!black},
  jump/.style={-Latex,line width=1.2pt,draw=violet!75!black},
  decision/.style={draw=black!45,fill=white,rounded corners=2pt,
    minimum width=.95cm,minimum height=.63cm,font=\small},
  implicated/.style={decision,draw=red!75!black,fill=red!9,line width=1pt},
  irrelevant/.style={decision,draw=black!22,fill=black!4,text=black!48},
  label/.style={fill=white,inner sep=2pt,font=\scriptsize}
]
  \fill[blue!4,rounded corners=5pt] (0,0) rectangle (7.85,5.15);
  \fill[violet!2,rounded corners=5pt] (8.35,0) rectangle (16.2,5.15);
  \draw[blue!30,rounded corners=5pt] (0,0) rectangle (7.85,5.15);
  \draw[violet!30,rounded corners=5pt] (8.35,0) rectangle (16.2,5.15);
  \node[font=\small\bfseries,anchor=west] at (.35,4.78)
    {(a) Chronological repair};
  \node[font=\small\bfseries,anchor=west] at (8.7,4.78)
    {(b) \core: learn and backjump};

  \node[implicated] (l1) at (.9,3.85) {$d_1$};
  \node[implicated] (l2) at (2.42,3.85) {$d_2$};
  \node[decision,fill=blue!6] (l3) at (3.94,3.85) {$d_3$};
  \node[decision,fill=blue!6] (l4) at (5.46,3.85) {$d_4$};
  \node[draw=red!70!black,fill=red!8,circle,minimum size=.68cm,
    text=red!70!black] (lf) at (7.02,3.85) {\faIcon{times}};
  \foreach \a/\b in {l1/l2,l2/l3,l3/l4,l4/lf}
    \draw[flow] (\a.east) -- (\b.west);

  \draw[repair] (lf.south) -- (7.02,3.08) -- (l4.south |- 0,3.08)
    -- (l4.south);
  \draw[repair] (l4.south) -- (5.46,2.36) -- (l3.south |- 0,2.36)
    -- (l3.south);
  \draw[repair] (l3.south) -- (3.94,1.64) -- (l2.south |- 0,1.64)
    -- (l2.south);
  \foreach \x/\y/\n in {6.25/3.08/1,4.7/2.36/2,3.18/1.64/3}
    \node[draw=blue!65!black,fill=blue!8,text=blue!65!black,circle,inner sep=1pt,
      minimum size=.3cm,font=\tiny] at (\x,\y) {\n};
  \node[draw=blue!35,fill=blue!5,rounded corners=3pt,
    minimum width=3.95cm,minimum height=.54cm,font=\scriptsize,
    text=blue!65!black] at (3.95,.57)
    {\faIcon{redo}\quad \faIcon{shield-alt}\enspace
     \faIcon{shield-alt}\enspace \faIcon{shield-alt}
     \quad repeated checks};

  \node[implicated] (r1) at (9.2,3.85) {$d_1$};
  \node[implicated] (r2) at (10.72,3.85) {$d_2$};
  \node[irrelevant] (r3) at (12.24,3.85) {$d_3$};
  \node[irrelevant] (r4) at (13.76,3.85) {$d_4$};
  \node[draw=red!70!black,fill=red!8,circle,minimum size=.68cm,
    text=red!70!black] (rf) at (15.32,3.85) {\faIcon{times}};
  \foreach \a/\b in {r1/r2,r2/r3,r3/r4,r4/rf}
    \draw[flow] (\a.east) -- (\b.west);
  \draw[black!35,densely dashed,rounded corners=3pt]
    (11.47,3.38) rectangle (14.52,4.3);
  \node[label,text=black!55] at (13,3.16) {skip $d_3,d_4$};

  \node[draw=red!65!black,fill=red!6,rounded corners=3pt,
    minimum width=3.15cm,minimum height=.64cm,font=\small]
    (conflict) at (13.52,2.34)
    {\faIcon{search}\quad $C=\{d_1,d_2\}$};
  \draw[-Latex,draw=red!70!black,line width=1pt]
    (rf.south) -- (15.32,2.83) -| (conflict.north);
  \draw[jump] (conflict.west) -- (10.72,2.34) -- (r2.south);
  \node[label,text=violet!75!black] at (10.48,2.14) {direct jump};

  \node[draw=orange!70!black,fill=orange!8,rounded corners=3pt,
    minimum width=2.65cm,minimum height=.57cm,font=\small]
    (cache) at (13.55,.77)
    {\faIcon{database}\quad cache $C$};
  \draw[-Latex,densely dashed,draw=orange!70!black,line width=.85pt]
    (conflict.south) -- (cache.north);
  \node[draw=black!25,fill=white,rounded corners=3pt,
    minimum width=2.25cm,minimum height=.57cm,font=\scriptsize,
    text=black!65] (prune) at (10.15,.77)
    {\faIcon{ban}\quad prune supersets};
  \draw[-Latex,densely dashed,draw=orange!70!black,line width=.85pt]
    (cache.west) -- (prune.east);
\end{tikzpicture}}
  \caption{The same failure under matched decision order. Chronological repair
  retreats through $d_4$ and $d_3$ before reconsidering $d_2$. \core uses
  $C=\{d_1,d_2\}$ to jump directly to $d_2$ and caches the conflict to prune
  its supersets. Gray decisions are absent from the core.}
  \label{fig:backjump}
\end{figure*}
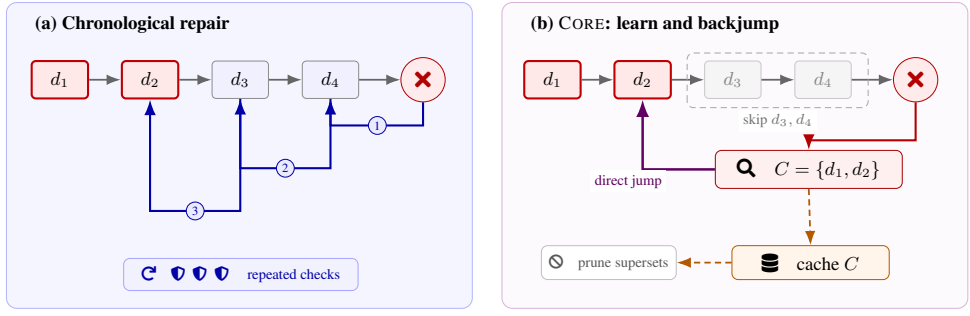

\subsection{Algorithm}

\begin{algorithm}[t]
\caption{\core search}
\label{alg:core}
\begin{algorithmic}[1]
\Require instance $x$, proposer $\propose$, verifier $\verify$,
         final scorer $\finalcheck$
\State $s \gets [\,]$; $\nogoods \gets \emptyset$
\While{budget remains}
  \If{$s$ is complete}
    \If{$\finalcheck(x,s)=1$}
      \State \Return $s$
    \EndIf
    \State $s\gets$\Call{RejectLeaf}{$s$}; \textbf{continue}
  \EndIf
  \State $A \gets \propose(x,s)$
  \State $d \gets$ first untried $a\in A$ such that
         $\neg\contains(s\oplus a,\nogoods)$
  \If{no such $d$ exists}
    \State $C \gets$ \Call{ExplainExhaustion}{$s$}
    \If{$C$ is certified}
      \State $(s,\nogoods)\gets$
        \Call{LearnAndJump}{$s$, $C$, $\nogoods$}
    \Else
      \State $s\gets$\Call{Backtrack}{$s$}
    \EndIf
    \State \textbf{continue}
  \EndIf
  \State $r\gets\verify(x,s\oplus d)$
  \If{$r=\mathsf{pass}$ \textbf{or} $r=\mathsf{unknown}$}
    \State $s\gets s\oplus d$
  \Else
    \State $(s,\nogoods)\gets$
      \Call{LearnAndJump}{$s\oplus d$, $r.C$, $\nogoods$}
  \EndIf
\EndWhile
\State \Return \textsc{BudgetExceeded}
\end{algorithmic}
\end{algorithm}

\core maintains a stack of active decisions and a set $\nogoods$ of learned
cores.  Before verifying a candidate, it checks whether the resulting state
contains a learned core.  Such a candidate is skipped without another verifier
call.  A new failure core is minimized when possible, stored, and used to
choose a backjump target.  An \textsf{unknown} result can extend the stack,
but only the final scorer can accept a complete state.

\textsc{LearnAndJump} first removes redundant members if the verifier supports
core minimization.  It inserts the remaining core into $\nogoods$, deleting
stored supersets.  For a nonempty core $C$ in $s_t$, the backjump level and
retained state are
\[
\begin{aligned}
j(C,s_t)&=\max\{i\leq t:\kappa(d_i)\in C\},\\
s'&=s_{j(C,s_t)-1}.
\end{aligned}
\]
The state is thus truncated to the level immediately before
$j=j(C,s_t)$, and the offending value at $j$ is marked tried.  Decisions
newer than $j$ are skipped:
they cannot resolve a conflict that does not mention them.  If all candidates
at a level are exhausted, their individual explanations are resolved into an
exhaustion core, analogous to eliminating the current variable in constraint
search \citep{robinson1965resolution,dechter1990enhancement}.
Figure~\ref{fig:backjump} contrasts this operation with chronological repair:
the conflict excludes an intermediate decision, so \core skips that level and
records a reusable nogood.

\paragraph{Candidate and core bookkeeping.}
The controller keeps a tried-candidate ledger for each retained prefix.
A candidate key includes its decision type, canonical value, dependencies,
and any identity needed to distinguish repeated choices; changing a dependency
changes the key even if the surface text is identical.
Within an instance, the cache test is
\[
\contains(s,\nogoods)
=\mathbf{1}\bigl[\exists C\in\nogoods:\ C\subseteq K(s)\bigr].
\]
A cached
core therefore rejects later supersets regardless of the order of decisions
outside $C$, without treating an unverified critique as a constraint.

\paragraph{Certifying a jump.}
Before learning $C$, the checker must establish that no valid completion
contains all of $C$.  Reproducing a failure on one complete trajectory is
insufficient: another continuation might succeed.  If core minimization is
available, each removed member is accepted only after the smaller core is
certified again; otherwise the last certified core is kept.  An exhaustion
core also requires coverage of every legal alternative at that decision
level.  If the complete alternative set is $A_j$, and each $a\in A_j$ has a
certified core $C_a$ containing $\kappa(a)$, a candidate exhaustion core is
\[
C_{\mathrm{exh}}
=\bigcup_{a\in A_j}\bigl(C_a\setminus\{\kappa(a)\}\bigr).
\]
Any completion retaining $C_{\mathrm{exh}}$ must choose an $a\in A_j$ and
therefore contain $C_a$.  Running out of sampled candidates in one proposal
batch does not
establish exhaustion, so the controller falls back to chronological repair
when the checker cannot certify it.
Under finite token or verifier-call budgets, the loop may instead return
\textsc{BudgetExceeded}; the completeness result below applies only when the
budget does not interrupt the search.

\paragraph{Worked 24-Game example.}
Consider the cards $\{1,3,4,6\}$.  The proposer first chooses
$d_1:1+3=4$, leaving the new $4$, the original $4$, and $6$.
It then chooses $d_2:4_{\mathrm{orig}}+6=10$ and
$d_3:4_{\mathrm{new}}+10=14$.  Before accepting $d_3$, an illustrative
search-time verifier rejects the terminal value 14 and uses exact
rational-arithmetic replay to explain the failure.  Exhaustive completion
from $\{4_{\mathrm{new}},
4_{\mathrm{orig}},6\}$ cannot yield 24, so the checker certifies
$\mathsf{fail}(C=\{d_1\})$: retaining $d_1$ makes every completion invalid,
regardless of $d_2$ and $d_3$.  \core caches this core and jumps from $d_3$
to $d_1$, skipping $d_2$.  A new first choice can then produce
$d'_1:3/4=3/4$, $d'_2:1-3/4=1/4$, and $d'_3:6/(1/4)=24$.
This illustrates how a certified prefix conflict identifies an earlier
decision to change, while the skipped decision is irrelevant to that
failure.

\subsection{Prompt and verifier contract}

An LLM instantiation should ask for one decision at a time in a machine-readable
schema, not an unrestricted replacement trace.  A minimal response contains
\texttt{id}, \texttt{value}, and \texttt{depends\_on}; optional prose remains
non-binding.  The controller, rather than the model, owns the stack, candidate
ledger, and learned cores.  On backjump, the prompt contains the retained
prefix and only the cores relevant to the next decision.  This avoids an
ever-growing transcript.

\textsc{RejectLeaf} marks a complete state rejected by the final scorer as
tried and retreats chronologically; it does not cache a conflict without a
sound explanation.  Likewise, exhaustion uses chronological backtracking
when no certified exhaustion core is available.

Conflict production has three practical tiers:
\begin{enumerate}
    \item \textbf{Exact:} SAT/SMT/CSP solvers, parsers, type checkers, and
    deterministic task constraints can return certified cores.
    \item \textbf{Replay-certified:} a proposed core is accepted only if a
    sound checker rules out every completion containing that subset.
    \item \textbf{Heuristic:} an LLM or learned process model names likely
    causes.  These may guide search but should be soft constraints with
    expiration or confirmation, because hard pruning can destroy completeness.
\end{enumerate}

\subsection{Guarantees}

\begin{theorem}[Safety and completeness]
\label{thm:complete}
Assume (i) every learned core is sound, (ii) branching and maximum reasoning
depth are finite,
(iii) the proposer eventually enumerates every legal candidate at a revisited
state, (iv) the search-time verifier never rejects a prefix of a valid
solution, (v) the final scorer accepts exactly the valid complete states, and
(vi) the run is not stopped by a finite budget before search completes.
Then \core never prunes a valid complete solution and returns a solution
whenever one exists.
\end{theorem}

\begin{proof}
A partial branch is pruned by a learned core only when it contains that core.
By core soundness, no valid solution contains the core.  A complete state
rejected by the final scorer is also invalid under the stated assumption.
A backjump removes only decisions newer than the latest member of the conflict;
changing any removed decision while retaining the whole core cannot produce a
solution.  Thus backjumping skips only invalid subtrees.  Finite branching and
eventual enumeration imply that every unpruned candidate is eventually tried.
If a solution exists, none of its prefixes contains a sound core, so its branch
is eventually reached and accepted by the final scorer.  Unknown partial
checks and chronological retreat without a certified core cannot remove it.
\end{proof}

\begin{proposition}[Termination]
Under the assumptions of Theorem~\ref{thm:complete}, \core terminates with a
solution or exhaustion.
\end{proposition}

\begin{proof}
There are finitely many decision sequences.  Candidate ledgers prevent retrying
the same decision at the same retained state, and learning can only remove
additional sequences.  Therefore the loop cannot visit infinitely many
distinct untried candidates.
\end{proof}

\paragraph{Why cores can save exponential work.}
Suppose a conflict depends on decisions at levels $i<j$ but is discovered at
depth $t$.  Chronological repair may enumerate combinations of levels
$j+1,\ldots,t$ before changing $j$.  A core that excludes those levels skips
their Cartesian product immediately.  Memoization also prunes the same
conflict when reached through a different ordering of irrelevant decisions.
This is a structural benefit, not a claim that every instance improves:
uninformative cores containing the full prefix reduce \core to chronological
backtracking plus bookkeeping.

\begin{figure}[t]
  \centering
  \includegraphics[width=0.8\linewidth]{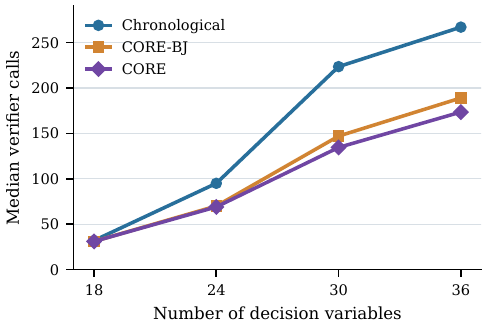}
  \caption{Median verifier calls over 500 paired instances per size.  All
  methods use the same proposal order and exact verifier.}
  \label{fig:calls}
\end{figure}

\begin{table}[t]
\centering
\small
\begin{tabular}{r rrr r}
\toprule
$n$ & Chron. & CORE-BJ & \core & $\Delta$ \\
\midrule
18 & 32.0 & 31.0 & \textbf{31.0} & $-3.1\%$ \\
24 & 95.0 & 70.0 & \textbf{69.0} & $-27.4\%$ \\
30 & 223.5 & 147.0 & \textbf{134.5} & $-39.8\%$ \\
36 & 267.0 & 189.0 & \textbf{173.5} & $-35.0\%$ \\
\bottomrule
\end{tabular}
\caption{Median verifier calls.  $\Delta$ compares \core with chronological
repair.  Every method solved every main-test instance within budget.}
\label{tab:main}
\end{table}

\section{Search Mechanism Analysis}

We use a controlled graph-coloring testbed to isolate the effect of the
search controller.  The evaluation on language-model reasoning tasks follows
in Section~\ref{sec:llm-eval}.

\begin{table*}[t]
\centering
\small
\resizebox{\textwidth}{!}{
\begin{tabular}{lrrrrrrrr}
\toprule
Method & \shortstack{GSM8K\\(500)} & \shortstack{MATH-500\\(500)} &
\shortstack{24-Game\\(200)} & \shortstack{Blocks\\(200)} &
\shortstack{HumanEval\\(164)} &
Mean & Verifier calls & Tokens (k) \\
\midrule
\multicolumn{9}{l}{\textit{Qwen2.5-7B-Instruct}} \\
Direct CoT & 72.4 & 42.8 & 61.0 & 41.0 & 58.5 & 55.1 & 1.0 & 1.3 \\
Self-Consistency-8 & 80.6 & 51.2 & 78.0 & 48.0 & 64.0 & 64.4 & 8.0 & 9.6 \\
Self-Refine-4 & 78.1 & 49.6 & 75.0 & 55.0 & 67.7 & 65.1 & 3.2 & 4.9 \\
Tree of Thoughts & 82.0 & 55.4 & 88.0 & 67.0 & 70.1 & 72.5 & 18.6 & 14.2 \\
\core & \textbf{83.8} & \textbf{58.6} & \textbf{92.0} &
\textbf{73.0} & \textbf{72.0} & \textbf{75.9} & 10.9 & 9.1 \\
\midrule
\multicolumn{9}{l}{\textit{Qwen3-8B}} \\
Direct CoT & 87.8 & 73.8 & 64.5 & 50.5 & 68.3 & 69.0 & 1.0 & 1.5 \\
Self-Consistency-8 & 92.2 & 79.8 & 79.5 & 54.5 & 72.0 & 75.6 & 8.0 & 6.5 \\
Self-Refine-4 & 89.8 & 78.2 & 76.5 & 59.5 & 72.0 & 75.2 & 3.0 & 4.8 \\
Tree of Thoughts & 93.2 & 82.8 & 88.5 & 69.5 & 75.0 & 81.8 & 19.0 & 7.0 \\
\core & \textbf{93.8} & \textbf{85.2} & \textbf{91.5} &
\textbf{74.5} & \textbf{76.2} & \textbf{84.2} & 12.0 & 6.0 \\
\bottomrule
\end{tabular}}
\caption{End-to-end task success rates (\%) for both backbones. Calls and
generated tokens are per-problem means. Mean success averages the five task
columns; sample counts are shown in the headers.}
\label{tab:llm-main}
\end{table*}

\subsection{Testbed}

We use planted 3-coloring because it exposes all relevant algorithmic objects
without an ambiguous learned judge.  For each size
$n\in\{18,24,30,36\}$, we generate 500 graphs.  Each vertex receives a planted
color; edges are sampled with probability $0.28$ only between differently
colored vertices, guaranteeing at least one solution.  Variables are ordered
by decreasing degree.  A deterministic noisy proposer ranks the planted color
first with probability $0.30$ and otherwise uses a seeded random order.  All
methods receive exactly the same candidate ordering.

The verifier accepts a color when it differs from every assigned neighbor.
Upon failure it returns the conflicting vertex assignments.  When every color
for a vertex fails, the search procedure resolves those explanations into the
set of prior assignments that block its domain.  These are exact conflict
cores.

\paragraph{Systems.}
\textsc{Chronological} is depth-first repair that returns one level after an
exhausted decision.  \textsc{CORE-BJ} uses conflict-directed backjumping but
does not cache cores.  \core adds core memoization.  The primary metric is the
number of verifier calls until the first solution; we also record expansions,
backtracks, jump distance, and learned cores.  Runs are capped at 100,000 calls.
All random choices are derived from recorded instance seeds.

\subsection{Results}

\begin{figure*}[t]
  \centering
  \includegraphics[width=0.8\linewidth]{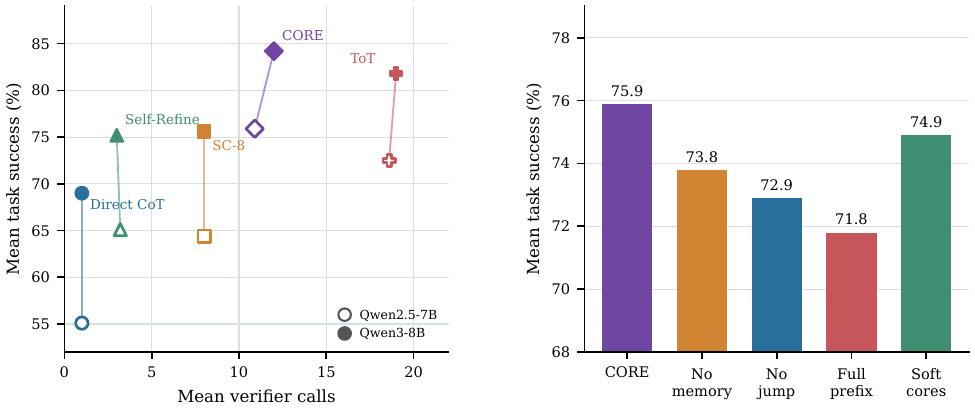}
  \caption{End-to-end results. Left: mean task success against mean verifier
  calls for both backbones. Right: Qwen2.5-7B-Instruct ablations of core
  memory, backjumping, and conflict precision.}
  \label{fig:llm}
\end{figure*}

\paragraph{Backjumping reduces verification.}
All three methods solve all 2,000 main instances.  Figure~\ref{fig:calls} and
Table~\ref{tab:main} show that the methods are similar on the smallest graphs,
where most runs encounter little backtracking.  The gap grows with deeper
search.  At 30 variables, \core reduces the median from 223.5 to 134.5 calls
(39.8\%); at 36 variables, it reduces 267.0 to 173.5 (35.0\%).  Mean calls at
$n=36$ fall from 546.0 to 235.4, indicating that conflict information is
especially useful on the heavy tail.

\paragraph{Memory adds value beyond jumping.}
\textsc{CORE-BJ} already improves substantially over chronological repair.
Memoization lowers median calls further at $n=24,30,36$ by 5.7\%, 8.5\%, and
8.2\%, respectively.  At $n=18$, little search repeats and the two variants
tie.  This pattern is consistent with the intended mechanism: learned cores
matter only after search begins to revisit a failed combination under
irrelevant surrounding decisions.

\paragraph{Sensitivity to proposal quality.}
At $n=30$, we rerun 200 paired instances with planted-first probabilities
$0.10$ and $0.60$.  With the weaker proposer, median calls are 236.0 for
chronological repair, 149.5 for backjumping, and 135.0 for \core.  With the
stronger proposer, they are 133.5, 92.0, and 90.5.  Conflict-directed search
helps in both regimes, while the incremental value of memory shrinks when a
good proposal reaches a solution before conflicts recur.

\section{Experiments}
\label{sec:llm-eval}

\subsection{Experimental setup}

\paragraph{Backbone and decoding.}
We report results with Qwen2.5-7B-Instruct and Qwen3-8B as proposers, both
served with vLLM using bfloat16 weights and the same decoding configuration.
Candidate decisions are sampled at temperature $0.7$ and
$\mathrm{top\_p}=0.95$, with at most 128 generated tokens per candidate.  The
controller requests four candidates per decision and permits at most 24
accepted decisions.  Per-problem budgets are 16,384 generated tokens and
32 verifier calls.

\paragraph{State and cache representation.}
The proposer emits JSON records with
\texttt{id}, \texttt{decision\_type}, \texttt{value}, and
\texttt{depends\_on}.  The controller canonicalizes numbers, variable names,
API arguments, and commutative expressions before hashing a decision together
with its type and dependencies.
Learned cores are stored as sorted sets of canonical decision--value hashes.
Subset lookup uses an inverted index from hashes to core identifiers.  Core
minimization performs deletion-based replay for at most eight additional
checker calls; if the cap is reached, the non-minimal sound core is retained.
Heuristic cores are never inserted into the hard cache.

\begin{table}[t]
\centering
\small
\resizebox{\columnwidth}{!}{
\begin{tabular}{lrrrr}
\toprule
Variant & Success & Calls & Core size & Jump \\
\midrule
\core & \textbf{75.9} & \textbf{10.9} & 2.7 & 3.8 \\
No core memory & 73.8 & 12.7 & 2.7 & 3.6 \\
No backjumping & 72.9 & 15.4 & 2.8 & 1.0 \\
Full-prefix conflicts & 71.8 & 16.1 & 8.9 & 1.2 \\
Soft heuristic cores & 74.9 & 14.3 & 3.4 & 2.9 \\
\bottomrule
\end{tabular}}
\caption{Qwen2.5-7B-Instruct ablations. Success is mean task accuracy (\%);
other columns are per-problem means. ``Jump'' is the number of decision levels
removed per repair.}
\label{tab:llm-ablation}
\end{table}

\paragraph{Tasks and sampling.}
For GSM8K \citep{cobbe2021gsm8k}, we sample 500 problems from the official
test split with fixed seed 1729 and retain the selected problem IDs.  We use
all 500 MATH-500 problems \citep{lightman2024verify}.  For 24-Game, we select
200 fixed puzzles from the Tree of Thoughts collection \citep{yao2023tree},
confirm solvability with an exhaustive solver, and retain each puzzle's four
cards and ID.  We select 200 Blocksworld instances from PlanBench
\citep{valmeekam2023planbench}, stratified by plan length.  We use all 164
HumanEval problems \citep{chen2021codex}.

\paragraph{Search-time checks and final scoring.}
\textbf{GSM8K:} Search checks arithmetic in structured equations and
consistency with declared dependencies.  Textual inferences that cannot be
checked mechanically return \textsf{unknown}.  Final scoring extracts the
numeric answer and compares it with the dataset answer.
\textbf{MATH-500:} Search checks explicit algebraic transformations and
substitution equalities; parsing failures return \textsf{unknown}.  Final
answers are compared using a fixed version of Math-Verify
\citep{mathverify}, with disputed parses reviewed manually.
\textbf{24-Game:} Exact rational arithmetic checks operands, operators,
results, and card-use counts during search.  Final scoring requires each of
the four cards exactly once and a value of 24.
\textbf{Blocksworld:} A deterministic state-transition checker verifies
action preconditions and successor states.  Final scoring executes the full
plan from the initial state and tests the goal.
\textbf{HumanEval:} Search checks syntax and runs a fixed set of public tests.
Test failures on unfinished programs return \textsf{unknown}.  Only the final
program is run against held-out evaluation tests, with generated code
executed in an isolated environment.

\paragraph{Baselines and budget matching.}
Direct CoT receives one 16,384-token attempt.  Self-Consistency samples eight
complete traces and votes over normalized answers.  Self-Refine performs up to
four critique--revision rounds.  Tree of Thoughts uses branching factor four,
beam width five, and the same verifier-call cap as \core.  Within each
backbone, all methods use the same proposer, prompt examples, answer parser,
and final verifier.  We report
realized token and verifier-call usage alongside task success because the
methods consume these resources differently.

\subsection{Results and analysis}

\paragraph{Main comparison.}
In Table~\ref{tab:llm-main}, \core has the highest success rate on all five
tasks with both backbones.  With Qwen2.5-7B-Instruct, mean success is 75.9\%,
versus 72.5\% for Tree of Thoughts.  With Qwen3-8B, it is 84.2\%, versus
81.8\%.  Blocksworld gives the largest gain over Tree of Thoughts for both
backbones ($+6.0$ and $+5.0$ points, respectively), where failed
preconditions identify earlier causal decisions.

\paragraph{Efficiency.}
Figure~\ref{fig:llm} shows the success--call tradeoff for both backbones.
Compared with Tree of Thoughts, \core uses 10.9 versus 18.6 calls and 9.1k
versus 14.2k generated tokens with Qwen2.5-7B-Instruct.  With Qwen3-8B, the
corresponding counts are 12.0 versus 19.0 calls and 6.0k versus 7.0k tokens.
Verifier calls alone do not measure wall-clock cost: structured generation
and core replay add work beyond a baseline answer check.

\paragraph{Ablations.}
The Qwen2.5-7B-Instruct ablations in Table~\ref{tab:llm-ablation} show that
removing core memory lowers mean task
success from 75.9\% to 73.8\% and raises mean verifier calls from 10.9 to
12.7.  Removing backjumping lowers success to 72.9\% and raises calls to 15.4.
Full-prefix conflicts yield the lowest success of the listed variants at
71.8\%.  Soft heuristic cores reach 74.9\%, but uncertified cores do not
provide the guarantee of hard, reusable nogoods.

\FloatBarrier
\paragraph{Task-level failure modes.}
In \textbf{GSM8K}, a misused quantity propagates through later arithmetic.
In \textbf{MATH-500}, a sign, algebra, or substitution error can implicate
an earlier equation.  In \textbf{24-Game}, an early operation can leave no
solution, and card reuse violates the rules.  In \textbf{Blocksworld}, an
early move can occupy a block or location needed later, with unrelated
actions in between.  In \textbf{HumanEval}, an interface choice can conflict
with a later function assumption.  These cases motivate revising the causal
decision instead of only the latest step.

\section{Conclusion}

\core treats a failed reasoning trajectory as more than a negative score.  A
small, valid explanation of failure is reusable search information: it says
which decision must change, which later decisions are irrelevant, and which
combination should never be tried again.  The resulting algorithm imports
conflict-directed backjumping and learning into a model-agnostic reasoning
interface, with guarantees that are explicit about verifier soundness.
Controlled experiments show reduced verification work under matched
proposals, and the five-task evaluation with both Qwen2.5-7B-Instruct and
Qwen3-8B reports higher task success and lower verifier-call use than Tree of
Thoughts.  Measuring latency and auditing core soundness are the next steps
for assessing these gains.

\bibliography{references}

\end{document}